\documentclass{article}

\usepackage[preprint]{neurips_2026}

\usepackage{amsthm}

\usepackage[utf8]{inputenc}
\usepackage[T1]{fontenc}
\usepackage{hyperref}
\usepackage{url}
\usepackage{booktabs}
\usepackage{amsfonts}
\usepackage{amsmath}
\usepackage{amssymb}
\usepackage{amsthm}
\newtheorem{proposition}{Proposition}

\usepackage{nicefrac}
\usepackage{microtype}
\usepackage[table]{xcolor}
\usepackage{graphicx}
\usepackage{subcaption}
\usepackage{algorithm}
\usepackage{algorithmic}
\usepackage{cleveref}
\usepackage{xspace}
\usepackage{multirow}
\usepackage{enumitem}
\usepackage{wrapfig}
\usepackage{listings}
\usepackage{tcolorbox}
\tcbuselibrary{listings,skins,breakable}
\usepackage{soul}
\usepackage{tikz}
\usetikzlibrary{positioning, arrows.meta, calc, backgrounds, fit}

\definecolor{kept}{RGB}{232,245,233}
\definecolor{pruned}{RGB}{255,235,238}
\definecolor{keptborder}{RGB}{76,175,80}
\definecolor{prunedborder}{RGB}{244,67,54}
\definecolor{codebg}{RGB}{250,250,250}
\definecolor{codeframe}{RGB}{200,200,200}
\definecolor{PosBlue}{RGB}{0,90,160}
\definecolor{NegBlue}{RGB}{140,170,200}
\definecolor{ctx}{RGB}{156,39,176}
\definecolor{sem}{RGB}{0,90,160}
\definecolor{dep}{RGB}{80,160,210}
\lstdefinestyle{codestyle}{
  basicstyle=\ttfamily\scriptsize,
  breaklines=true,
  columns=fullflexible,
  keepspaces=true,
  showstringspaces=false,
  tabsize=4,
  xleftmargin=2em,
  numbers=left,
  numberstyle=\tiny\color{gray},
  numbersep=6pt,
}

\newcommand{\method}{LaMR\xspace}

\title{Context Pruning for Coding Agents via Multi-Rubric Latent Reasoning}

\author{%
  Anonymous Author(s)
}

\author{%
Jingjing Wang\thanks{Equal contribution.}\    \footnotemark[2] \\
  Clemson University \\
  \texttt{jingjiw@clemson.edu} \\
   \And
  Xiwen Chen\footnotemark[1]\ \thanks{Corresponding authors.}  \\
  Morgan Stanley \\
  \texttt{xiwenchen.cs@outlook.com} \\
  \And
  Wenhui Zhu\footnotemark[1] \\
  Arizona State University \\
  \texttt{wzhu59@asu.edu} \\
  \AND
  Huayu Li   \\
 \quad  University of Arizona \\
  \quad  \texttt{hl459@arizona.edu} \\
  \And
  \quad \quad Zhengxiao He \\
   \quad  \quad \quad University of Arizona \\
   \quad \quad \texttt{zhengxiaohe@arizona.edu} \\
  \And
  \quad \quad Feiyang Cai \\
  \quad \quad  Clemson University \\
 \quad \quad  \texttt{feiyang@clemson.edu} \\
  \And
 \hspace{-2em} Ana S. Carreon-Rascon \\
  \hspace{-2em} University of Arizona \\
  \hspace{-2em} \texttt{anascarreonr@arizona.edu} \\
  \And
   \quad Xuanzhao Dong \\
  \quad  Arizona State University \\
   \quad \texttt{xdong64@asu.edu} \\
  \And
  \quad \quad Feng Luo \\
 \quad \quad Clemson University \\
  \quad \quad \texttt{luofeng@clemson.edu} \\
}

\begin{document}

\maketitle

\begin{abstract}
LLM-powered coding agents spend the majority of their token budget reading repository files, yet much of the retrieved code is irrelevant to the task at hand. Existing learned pruners compress this context with a single-objective sequence labeler, collapsing all facets of code relevance into one score and one transition matrix. We show that this formulation creates a modeling bottleneck: a single CRF transition prior must serve heterogeneous retention patterns, including contiguous semantic spans and sparse structural support lines. We propose \textbf{\method{} (Latent Multi-Rubric)}, a structured pruning framework that decomposes code relevance into two interpretable quality dimensions, semantic evidence and dependency support, each modeled by a dedicated CRF with dimension-specific transition dynamics. A mixture-of-experts gating network dynamically weights the per-rubric emissions conditioned on the query, and a final CRF layer on the fused emissions produces the aggregate keep-or-prune decision. To supervise each dimension without additional annotation cost, we derive multi-rubric labels from the existing training corpus via AST-based program analysis, simultaneously denoising the teacher's binary labels. By effectively filtering distracting noise, \method{} frequently matches or even outperforms unpruned full-context baselines. Experiments on four benchmarks (SWE-Bench Verified, SWE-QA, LCC, LongCodeQA) show that \method{} wins 12 of 16 head-to-head multi-turn comparisons. It saves up to 31\% more tokens on multi-turn agent tasks and improves Exact Match by up to +3.5 on single-turn tasks, while performance is frequently enhanced by denoising the context, and any remaining drops are marginal.

\end{abstract}

\begin{figure}[h]
\centering
\includegraphics[width=\textwidth]{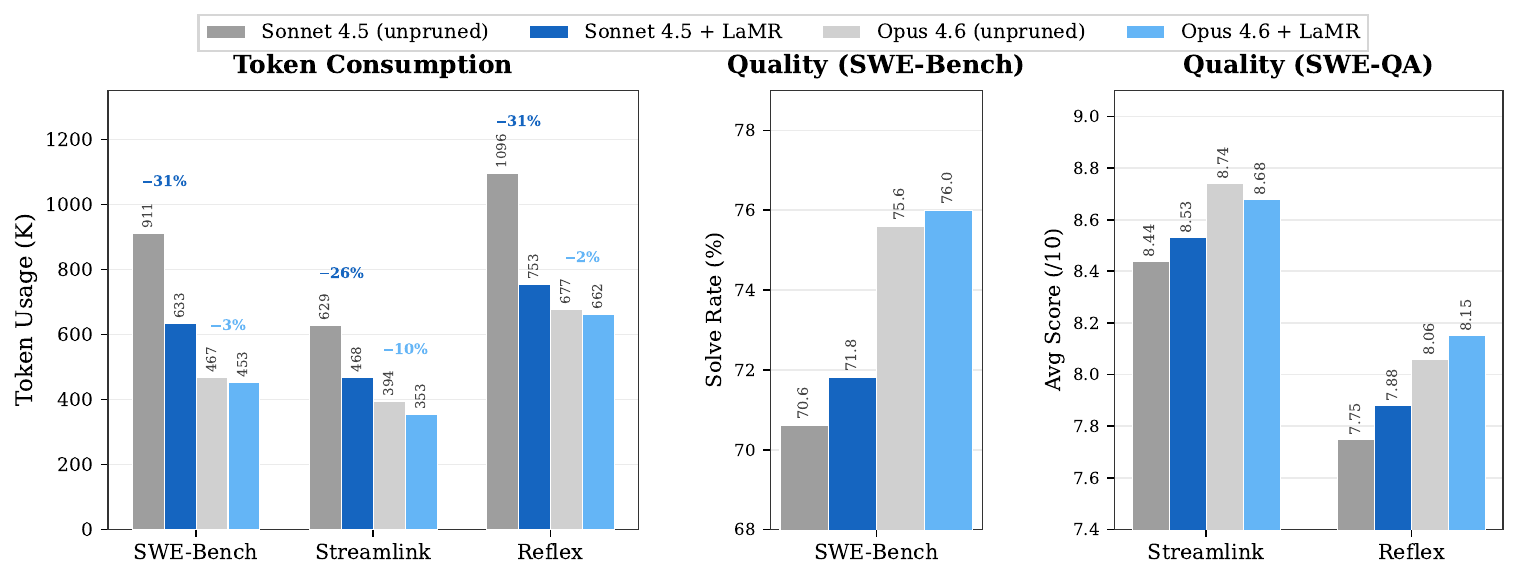}
\caption{Effect of \method across two backbone models and three benchmarks.
\textbf{Left}: Token consumption (K tokens);
percentage labels show reduction by \method relative to unpruned.
\textbf{Middle/Right}: Task quality
(SWE-Bench solve rate in \%;
SWE-QA average score on a 10-point scale).
Gray bars are unpruned baselines; blue bars are with \method.
\method consistently reduces token usage
while maintaining or slightly improving quality.}
\label{fig:main_results}
\end{figure}

\section{Introduction}
\label{sec:introduction}

LLM-powered coding agents are rapidly reshaping software engineering, with state-of-the-art (SOTA) frameworks such as SWE-agent \citep{yang2024sweagent}, OpenHands \citep{wang2024openhands}, and Agentless \citep{xia2024agentless} automating complex, repository-level tasks. Commercial success followed suit: Anthropic's Claude Code and Cursor have reached multi-billion dollar valuations by processing millions of lines of code daily. However, these agents have hit a \textit{context wall}: they spend approximately 67--76\% of their token budget simply reading files \citep{swepruner}. This context bloat not only incurs prohibitive API costs but also triggers "attention dilution" and the \textit{lost-in-the-middle} phenomenon \citep{liu2023lost}, where irrelevant code segments degrade the agent's reasoning precision.

The quest for context efficiency has evolved through three distinct paradigms. Initially,  task-agnostic heuristics were employed, such as perplexity-based token pruning \citep{jiang2023llmlingua, li2023compressing} or similarity-based RAG \citep{gao2023retrieval}. However, as noted by \citep{shi2025longcodezip}, these methods primarily target natural language and face critical limitations when applied to code: token-level pruning frequently compromises syntactic validity, while abstractive techniques discard character-level details essential for debugging. Furthermore, these approaches operate with static compression ratios and task-agnostic criteria, failing to adapt to the evolving informational needs of multi-turn agent interactions \citep{ yang2024sweagent}.

Recently, SWE-Pruner \citep{swepruner} mitigated this issue by framing code context pruning as a structured sequence labeling task. It trains a 0.6B-parameter neural skimmer by minimizing the conditional random field negative log likelihood (CRF-NLL) \citep{zheng2015conditional}. This allows adaptive token usage reduction while keeping character-level integrity. However, this single-objective formulation introduces an inherent conflict: it uses one transition matrix to model all types of code relevance. Code relevance operates across distinct dimensions. Semantic relevance usually forms continuous blocks, which require high self-transition probabilities in the CRF. In contrast, structural dependencies (like \texttt{import} statements) and control-flow pairs (like \texttt{try}/\texttt{except}) are scattered. A single transition matrix creates a modeling bottleneck: it must serve both dense semantic spans (contiguous keep-blocks) and sparse structural support (scattered single-line jumps) with one set of parameters. 
Forced to compromise, the model biases toward the denser semantic blocks. It retains the main function body but discards necessary structural lines (\cref{fig:motivation}). This missing structure severely impacts strong coding agents. When models like Claude~4.6 encounter broken syntax or missing context, they fail and repeatedly retry the task. These iterative retries end up consuming more tokens than the unpruned baseline.

To address this, we propose \method{} (\textbf{La}tent \textbf{M}ulti-\textbf{R}ubric). We decompose code relevance into \(K\) latent rubrics, each modeled by a dedicated CRF head with its own transition matrix. This allows dimension-specific sequence patterns. A query-adaptive Mixture-of-Experts (MoE) gate then combines the per-dimension CRF emissions, dynamically balancing semantic and structural signals for the current task. Since existing datasets provide only binary keep-or-prune masks, we introduce Rubric-Guided Labeling to derive fine-grained supervision without additional annotation. Using Abstract Syntax Tree (AST) analysis, we extract dimension-specific labels from the existing masks and recover structurally necessary lines that the teacher may omit, such as imports, class headers, and control-flow companions. This produces a cleaner training signal for the neural skimmer. More broadly, \method{} is not merely deleting tokens; it aims to preserve self-contained evidence-support units: task-relevant lines together with the structural context needed to interpret them. By effectively extracting signal from noise, \method{} provides agents with higher-quality observations. We empirically demonstrate that this structure-aware approach not only compresses the context but carries the potential to elevate the agent's final task performance. 

We summarize our contributions below: (i)\textbf{Latent Multi-Rubric Formulation.} We identify the transition conflict in single-objective pruning and resolve it by modeling code relevance through $K$ separate CRF heads.
  (ii) \textbf{Query-Adaptive Gating.} We use an MoE gating network to adjust the focus between semantic and structural code based on the agent's query.
  (iii) \textbf{AST-based Label Denoising.} We develop a zero-cost pipeline to extract multi-dimensional labels from binary masks, which corrects missing structural tokens in the teacher data.
  (iv) \textbf{Superior Efficiency-Performance Tradeoff.} We extensively evaluate \method across four benchmarks and two state-of-the-art LLM backbones. As a lightweight middleware, it consistently mitigates token bloat and achieves a strictly better tradeoff between context compression and downstream task quality compared to prior methods.

\begin{figure}[h]
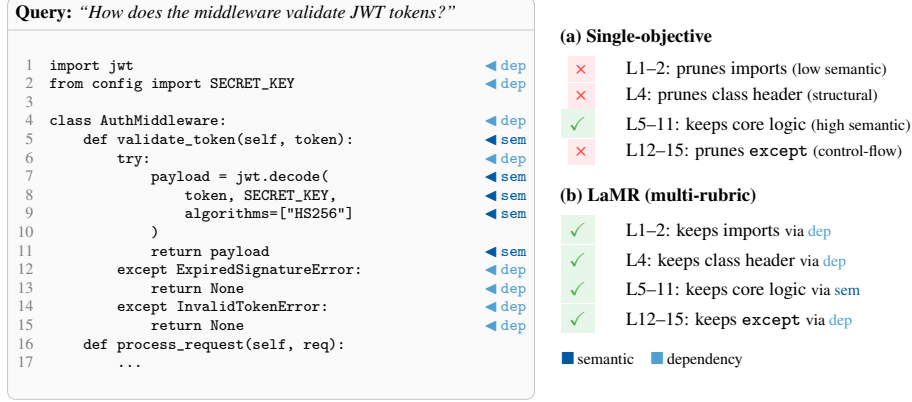

\centering
\small
\begin{minipage}[c]{0.50\textwidth}
\begin{tcolorbox}[
  colback=codebg, colframe=codeframe,
  coltitle=black, colbacktitle=codebg,
  boxrule=0.4pt, left=0pt, right=0pt, top=1pt, bottom=1pt,
  title={\scriptsize\textbf{Query:} \emph{``How does the middleware validate JWT tokens?''}},
  fonttitle=\scriptsize,
]
\begin{lstlisting}[style=codestyle, escapechar=@, basicstyle=\ttfamily\tiny]
import jwt                           @\hfill\textcolor{dep}{\tiny$\blacktriangleleft$\,dep}@
from config import SECRET_KEY        @\hfill\textcolor{dep}{\tiny$\blacktriangleleft$\,dep}@

class AuthMiddleware:                @\hfill\textcolor{dep}{\tiny$\blacktriangleleft$\,dep}@
    def validate_token(self, token): @\hfill\textcolor{sem}{\tiny$\blacktriangleleft$\,sem}@
        try:                         @\hfill\textcolor{dep}{\tiny$\blacktriangleleft$\,dep}@
            payload = jwt.decode(    @\hfill\textcolor{sem}{\tiny$\blacktriangleleft$\,sem}@
                token, SECRET_KEY,   @\hfill\textcolor{sem}{\tiny$\blacktriangleleft$\,sem}@
                algorithms=["HS256"] @\hfill\textcolor{sem}{\tiny$\blacktriangleleft$\,sem}@
            )
            return payload           @\hfill\textcolor{sem}{\tiny$\blacktriangleleft$\,sem}@
        except ExpiredSignatureError:@\hfill\textcolor{dep}{\tiny$\blacktriangleleft$\,dep}@
            return None              @\hfill\textcolor{dep}{\tiny$\blacktriangleleft$\,dep}@
        except InvalidTokenError:    @\hfill\textcolor{dep}{\tiny$\blacktriangleleft$\,dep}@
            return None              @\hfill\textcolor{dep}{\tiny$\blacktriangleleft$\,dep}@
    def process_request(self, req):
        ...
\end{lstlisting}
\end{tcolorbox}
\end{minipage}%
\hfill
\begin{minipage}[c]{0.47\textwidth}
\scriptsize
\textbf{(a) Single-objective}\\[3pt]
\begin{tabular}{@{}rl@{}}
\colorbox{pruned}{\textcolor{prunedborder}{\texttimes}} & L1--2: prunes imports {\tiny(low semantic)} \\[1pt]
\colorbox{pruned}{\textcolor{prunedborder}{\texttimes}} & L4: prunes class header {\tiny(structural)} \\[1pt]
\colorbox{kept}{\textcolor{keptborder}{\checkmark}} & L5--11: keeps core logic {\tiny(high semantic)} \\[1pt]
\colorbox{pruned}{\textcolor{prunedborder}{\texttimes}} & L12--15: prunes \texttt{except} {\tiny(control-flow)} \\
\end{tabular}
\\[8pt]
\textbf{(b) \method{} (multi-rubric)}\\[3pt]
\begin{tabular}{@{}rl@{}}
\colorbox{kept}{\textcolor{keptborder}{\checkmark}} & L1--2: keeps imports {\tiny via \textcolor{dep}{dep}} \\[1pt]
\colorbox{kept}{\textcolor{keptborder}{\checkmark}} & L4: keeps class header {\tiny via \textcolor{dep}{dep}} \\[1pt]
\colorbox{kept}{\textcolor{keptborder}{\checkmark}} & L5--11: keeps core logic {\tiny via \textcolor{sem}{sem}} \\[1pt]
\colorbox{kept}{\textcolor{keptborder}{\checkmark}} & L12--15: keeps \texttt{except} {\tiny via \textcolor{dep}{dep}} \\
\end{tabular}
\\[6pt]
{\tiny
\textcolor{sem}{$\blacksquare$}\,semantic\quad
\textcolor{dep}{$\blacksquare$}\,dependency
}
\end{minipage}
\vspace{-2pt}
\caption{
  A single-objective pruner collapses semantic and structural
  relevance into one score,
  discarding lines that rank low on semantics alone.
  \method decomposes relevance into two rubrics
  (\textcolor{sem}{semantic} and \textcolor{dep}{dependency}),
  each with a dedicated CRF,
  and fuses them via query-adaptive gating.
}
\label{fig:motivation}
\end{figure}

\section{Preliminaries}
\label{sec:preliminaries}

\noindent\textbf{Code context pruning.}\quad
Given a code context $C = (x_1, \dots, x_n)$ of $n$ tokens
and a natural-language query $q$ describing the agent's current information need,
the task is to produce a binary label sequence
$\mathbf{y} \in \{\texttt{keep}, \texttt{prune}\}^n$
such that the retained tokens preserve query-relevant information
while minimizing context length.
Token-level scores are aggregated to line granularity
and thresholded to produce the final pruned output.

\noindent\textbf{From heuristic to learned pruning.}\quad
Prior approaches rely on task-agnostic heuristics:
LLMLingua~\citep{jiang2023llmlingua} and
Selective-Context~\citep{li2023compressing}
prune tokens based on perplexity or self-information,
while RAG-based methods retrieve code chunks
via embedding similarity~\citep{guo2022unixcoder}.
These methods use surface-level statistics
without understanding code structure or the agent's intent.
SWE-Pruner~\citep{swepruner} takes a fundamentally different approach:
it trains a lightweight neural skimmer
that \emph{learns} query-conditioned pruning
as structured sequence labeling with a linear-chain
CRF~\citep{zheng2015conditional}.
A CRF models the conditional probability
of the full label sequence jointly,
rather than predicting each token independently:
\begin{equation}
  P(\mathbf{y} \mid \mathbf{x}) =
    \frac{1}{Z(\mathbf{x})}
    \exp\!\Bigl(
      \sum_{t=1}^{n} \phi(y_t, \mathbf{x}, t)
      + \sum_{t=2}^{n} \psi(y_{t-1}, y_t)
    \Bigr),
  \label{eq:crf_def}
\end{equation}
where $\phi(y_t, \mathbf{x}, t)$ are \emph{emission potentials}
that score how likely token $t$ is to be kept or pruned
based on its content alone,
and $\psi(y_{t-1}, y_t)$ are \emph{transition potentials}
stored in a learnable $2\!\times\!2$ matrix
that bias adjacent tokens toward label consistency
(e.g., preferring consecutive keeps over isolated ones).
$Z(\mathbf{x})$ is a normalizing constant
computed via the forward algorithm.
This structured formulation avoids fragmented pruning
and achieves 23--38\% token reduction
on SWE-Bench Verified while improving task success rates.

\noindent\textbf{Limitations of single-objective pruning.}\quad
Despite its effectiveness,
the single-objective CRF formulation has an inherent limitation:
it relies on one set of emission scores and one transition matrix,
collapsing multiple aspects of code relevance into a single scalar.
As illustrated in \Cref{fig:motivation},
when a query targets token validation logic,
the model correctly identifies the core function body as relevant
but discards \texttt{import} statements (low semantic score, but structurally necessary), class headers (structural scaffolding), and \texttt{except} clauses (low semantic relevance,
but contextually paired with the \texttt{try} block).
These failures stem from a fundamental tension:
a single transition matrix must simultaneously model
(i)~contiguous semantic blocks
and (ii)~non-contiguous dependency chains
(including paired control-flow structures). These failures reflect a modeling bottleneck: one transition prior must serve both contiguous semantic spans and sparse dependency-support lines. 
Because the model does not explicitly distinguish why a line should be retained, it can favor dense semantic blocks while under-preserving structural support needed to interpret them.

\section{Method}
\label{sec:method}

\begin{tcolorbox}[
  colback=blue!3, colframe=blue!40!black,
  boxrule=0.5pt, arc=2pt,
  left=6pt, right=6pt, top=4pt, bottom=4pt,
  title={\small\textbf{Key Insight}},
  fonttitle=\small\bfseries,
]
\small
Code relevance is not a single quantity.
A line may be retained for its \emph{semantic} match to the query,
or for its role in a \emph{structural dependency} chain
(imports, scope headers, paired control flow).
These aspects exhibit distinct sequential dynamics
that a single CRF transition matrix cannot accommodate.
\method assigns each aspect its own CRF head
with an independent transition matrix,
and fuses them via a query-adaptive gating network.
\end{tcolorbox}

\vspace{2pt}


We decompose code relevance
into $K$ quality dimensions
$\mathcal{O} = \{o_1, \dots, o_K\}$,
each capturing a distinct reason to retain a line (\Cref{fig:motivation}).
In practice we use
$\mathcal{O} = \{\text{semantic}, \text{dependency}\}$ ($K\!=\!2$):
\emph{semantic} captures topical relevance to the query,
while \emph{dependency} captures structural necessity
(imports, definitions, scope headers,
and control-flow siblings such as \texttt{try}/\texttt{except}).
We find that these two dimensions
already subsume the contextual-coherence signals
illustrated in \Cref{fig:motivation},
since AST-based dependency tracing
naturally propagates through enclosing scopes
and paired control-flow branches
(\S\ref{sec:rubric_labeling}).
We call these dimensions \emph{latent rubrics}:
they are induced during training
via supervision derived from the teacher LLM's retention decisions,
conditioned on static program analysis (\S\ref{sec:rubric_labeling}),
but at inference the model produces per-rubric emissions
entirely from learned parameters,
without requiring program analysis for the scoring itself
(AST is only used in a lightweight post-processing step;
see \S\ref{sec:inference}).
The model first extracts a shared token representation
via multi-layer feature fusion (\S\ref{par:fusion}),
then projects it into per-dimension emissions
fed to $K$ independent CRF heads (\S\ref{par:crf_heads}),
whose outputs are combined by an MoE gating network (\S\ref{par:gating})
and decoded through a final fused CRF (\S\ref{par:fused_crf}).
We describe each component below.

\noindent\textbf{Multi-layer feature fusion.}\label{par:fusion}\quad
Following SWE-Pruner, we adopt the same backbone encoder
(a 0.6B-parameter reranker based on Qwen3-Reranker-0.6B ~\citep{zhang2025qwen3} with $L$ transformer layers).
The query $q$ and code context $C$ are concatenated
into a single input sequence,
so all hidden representations are inherently query-conditioned.
Let $\mathbf{H}^{(l)} \in \mathbb{R}^{n \times d}$
denote the hidden-state matrix at the $l$-th layer.
To provide richer per-token representations,
we extract hidden states from three layers
(an early layer $l_e$, a middle layer $l_m$,
and the final layer $L$), concatenate them,
and refine the result through
$N_\text{fuse}$ multi-head self-attention blocks
with residual connections:
\begin{align}
  \mathbf{H}^{\text{cat}} &=
    \bigl[\,
      \mathbf{H}^{(l_e)}
      \;\|\;
      \mathbf{H}^{(l_m)}
      \;\|\;
      \mathbf{H}^{(L)}
    \,\bigr]
    \in \mathbb{R}^{n \times 3d},
  \label{eq:fusion_concat} \\[2pt]
  \mathbf{H}^{(i+1)} &=
    \mathrm{LN}\!\bigl(\mathbf{H}^{(i)} + \mathrm{MHA}(\mathbf{H}^{(i)})\bigr),
  \quad i = 0, \dots, N_\text{fuse}\!-\!1,
  \label{eq:fusion_mha}
\end{align}
where $\|$ denotes concatenation
and $\mathbf{H}^{(0)} = \mathbf{H}^{\text{cat}}$.
In practice, $l_e = \lfloor 0.25 L \rfloor$
and $l_m = \lfloor 0.5 L \rfloor$
(layers 7 and 14 for the 28-layer backbone).
The resulting per-token representation
$\mathbf{h}_t \in \mathbb{R}^{3d}$
serves as input to both the emission heads and the gating network.

\noindent\textbf{Per-objective CRF heads.}\label{par:crf_heads}\quad
A shared MLP maps the fused representations
to per-objective emission vectors
$\mathbf{E} = \mathrm{MLP}(\mathbf{h}) \in \mathbb{R}^{n \times K \times 2}$,
where $\mathbf{E}_{t,k} \in \mathbb{R}^2$ scores
how likely token $t$ should be kept or pruned
from the perspective of rubric dimension $k$.
Crucially, each dimension $k$ maintains
its own learnable CRF parameters:
a transition matrix
$\mathbf{T}^{(k)} \in \mathbb{R}^{2 \times 2}$
encoding which label transitions
(keep$\to$keep, keep$\to$prune, etc.)
are natural for that dimension,
along with start potentials
$\boldsymbol{\pi}^{(k)} \in \mathbb{R}^2$
and end potentials
$\boldsymbol{\omega}^{(k)} \in \mathbb{R}^2$
that bias the first and last token's label respectively.
The path score under dimension $k$ is:
\begin{equation}
  S_k(\mathbf{x}, \mathbf{y}) =
    \pi^{(k)}_{y_1}
    + \sum_{t=1}^{n} \mathbf{E}_{t,k,y_t}
    + \sum_{t=2}^{n} \mathbf{T}^{(k)}_{y_t, y_{t-1}}
    + \omega^{(k)}_{y_n}.
  \label{eq:crf_score}
\end{equation}
Intuitively, the emission term asks
``is this token relevant under dimension $k$?''
while the transition term asks
``given the previous decision,
should the label change here?''
By design, different dimensions can learn
qualitatively different transition behaviors.
We expect \emph{semantic} to favor long runs of the same label
(topically related code is contiguous),
while \emph{dependency} tolerates frequent alternation
(imports, scope headers, and their usage sites are scattered;
paired structures like \texttt{try}/\texttt{except}
are also captured through dependency tracing). 
A single transition matrix cannot accommodate
both regimes simultaneously,
which is precisely the limitation we address.

\noindent\textbf{Mixture-of-experts gating.}\label{par:gating}\quad
Queries with different intents require different emphasis
across the $K$ objectives.
A lightweight gating network produces per-token weights
$\mathbf{g}_t
  = \mathrm{Softmax}\!\bigl(\mathrm{MLP}_\text{gate}(\mathbf{h}_t)\bigr)
  \in \Delta^{K-1}$,
and the fused emission at each position is:
\begin{equation}
  \mathbf{e}_t^{\text{fused}}
    = \sum_{k=1}^{K} g_{t,k} \;\mathbf{E}_{t,k}
    \in \mathbb{R}^2.
  \label{eq:fused_emission}
\end{equation}
To prevent gate collapse onto a single dominant objective,
we regularize with an entropy penalty
$\mathcal{L}_{\text{gate}}
  = \log K - \mathbb{H}[\bar{\mathbf{g}}]$,
where
$\bar{\mathbf{g}} = |\mathcal{V}|^{-1}\sum_{t \in \mathcal{V}} \mathbf{g}_t$
is the mean gate vector over valid positions
and $\mathbb{H}[\cdot]$ denotes Shannon entropy.

\noindent\textbf{Fused CRF.}\label{par:fused_crf}\quad
The per-objective CRFs model dimension-specific
transition dynamics during training,
but the final pruning decision must be a single coherent sequence.
When $K > 1$,
the gated emissions $\{\mathbf{e}_t^{\text{fused}}\}_{t=1}^n$
are therefore passed through a \emph{final CRF layer}
with its own transition matrix $\mathbf{T}^{\text{f}} \in \mathbb{R}^{2 \times 2}$
and start/end potentials.
This layer captures cross-dimension sequential dependencies
that may not be expressible
within any single rubric's transition matrix alone.
The corresponding loss is:
\begin{equation}
  \mathcal{L}_{\text{main}}
    = \frac{1}{n}
      \bigl(\log Z_\text{f}(\mathbf{x}) - S_\text{f}(\mathbf{x}, \mathbf{y})\bigr),
  \label{eq:main_crf}
\end{equation}
where $\mathbf{y}$ is the aggregate binary label sequence.

\noindent\textbf{Document-level scoring.}\quad
As an auxiliary task that encourages the backbone
to maintain a global view of document relevance
alongside fine-grained token pruning,
we preserve the backbone reranker's
document-level relevance estimation.
The last token's hidden state is projected onto the vocabulary embeddings
to extract the log-probability of ``yes'' relative to ``no'':
$s_{\text{doc}}
  = \log\sigma(\mathbf{h}_n^\top \mathbf{w}_{\texttt{yes}}
    - \mathbf{h}_n^\top \mathbf{w}_{\texttt{no}})$.
This score is supervised with
$\mathcal{L}_{\text{score}} = \mathrm{MSE}(e^{s_\text{doc}}, s_\text{ref})$,
where $s_\text{ref} \in [0,1]$ is a teacher-provided relevance score.

\subsection{Training Data and Rubric-Guided Labeling}
\label{sec:rubric_labeling}

\begin{figure}[htbp]
\centering
\resizebox{\textwidth}{!}{
\begin{tikzpicture}[
    node distance=0.8cm and 0.8cm,
    base/.style={draw, thick, align=center, rounded corners=2pt, font=\scriptsize, inner sep=4pt, fill=white},
    module/.style={base, minimum width=2.4cm, minimum height=0.6cm},
    head/.style={base, minimum width=1.5cm, minimum height=0.5cm},
    output/.style={base, fill=black, text=white, minimum width=2.6cm},
    workflow_frame/.style={draw, dashed, thick, gray!40, rounded corners=4pt, fill=gray!2}
]

    \node (agent) [module, fill=blue!5] {\textbf{Agent (Claude)} \\ Agentic Policy};

    \node (input) [module, right=0.8cm of agent] {Query $q$ + Code $C$};
    
    \node (backbone) [module, fill=gray!10, right=0.6cm of input] {Shared Backbone \\ Feature Fusion};
    
    \node (gating) [base, fill=orange!10, right=1.0cm of backbone] {MoE Gate};
    \node (sem_crf) [head, fill=blue!5, above=0.3cm of gating] {$\text{CRF}_{\text{sem}}$};
    \node (dep_crf) [head, fill=green!5, below=0.3cm of gating] {$\text{CRF}_{\text{dep}}$};
    
    \node (fused) [module, fill=purple!5, right=1.0cm of gating] {Fused CRF \\ Viterbi Decoding};

    \node (out) [output, right=0.8cm of fused] {\textbf{Pruned Context}};

    \begin{pgfonlayer}{background}
        \node (lamr_box) [workflow_frame, fit=(input) (backbone) (gating) (sem_crf) (dep_crf) (fused)] {};
        \node [font=\footnotesize\bfseries, gray, anchor=south west] at (lamr_box.north west) {LaMR Workflow};
    \end{pgfonlayer}

    \draw [-Stealth, thick] (agent) -- (input);
    \draw [-Stealth, thick] (input) -- (backbone);
    
    \draw [-Stealth, thick, dashed, gray] (backbone.east) -- (gating.west);
    
    \draw [-Stealth, thick] (backbone.east) -- ++(0.4,0) |- (sem_crf.west);
    \draw [-Stealth, thick] (backbone.east) -- ++(0.4,0) |- (dep_crf.west);
    
    \draw [-Stealth, thick] (sem_crf.east) -- ++(0.4,0) |- (fused.west);
    \draw [-Stealth, thick] (dep_crf.east) -- ++(0.4,0) |- (fused.west);
    
    \draw [-Stealth, thick, orange] (gating.east) -- node[above=4pt, pos=0.275, font=\tiny, text=black] {Weights $g_t$} (fused.west);
    \draw [-Stealth, thick] (fused) -- (out);

    \draw [-Stealth, thick, gray, rounded corners=6pt] 
        (out.south) -- ++(0, -1.2) -| 
        node[pos=0.3, above=6pt, font=\scriptsize\itshape, align=center, text=black] {Syntactically Valid Context} 
        (agent.south);

\end{tikzpicture}
}
\caption{\textbf{The LaMR workflow.} Operating as an agentic middleware, it intercepts file reads, routes features through parallel latent rubrics, and returns a syntactically pruned context via the loop.}
\label{fig:lamr_workflow}
\end{figure}
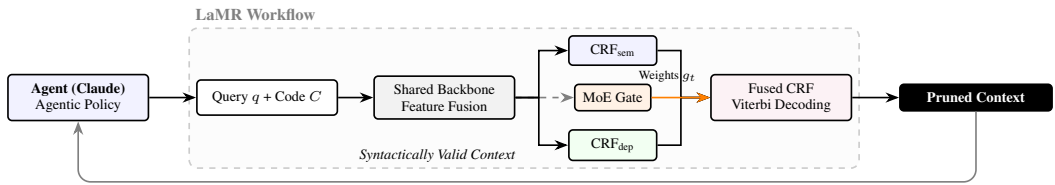

\noindent\textbf{Base training corpus.}\quad
We build on the training set constructed by
SWE-Pruner~\citep{swepruner}:
code snippets sampled from high-quality GitHub repositories
are paired with task-oriented queries synthesized by a teacher LLM.
Each example contains a line-level retention mask $\mathbf{y}$
and a document-level relevance score $s_{\mathrm{ref}}$.
An LLM-as-a-Judge filter retains high-quality examples,
yielding approximately 61K training quadruplets
$(q, C, \mathbf{y}, s_{\mathrm{ref}})$.
These aggregate labels supervise the main pruning loss
$\mathcal{L}_{\mathrm{main}}$
and the relevance-score loss
$\mathcal{L}_{\mathrm{score}}$.
However, the aggregate mask collapses different reasons for retention
into a single binary decision: a line may be kept because it is direct
query evidence, or because it structurally supports another retained line.
Our deployed sem+dep model is trained on a rubric-enriched extension
of this corpus. For each example, we construct per-line rubric labels
$\{\mathbf{y}^{(o)}\}_{o\in\mathcal{O}}$
using the retained-line mask together with static program analysis
.
Although the rubric builder can produce multiple diagnostic dimensions,
the deployed checkpoint uses the name-projected
semantic and dependency labels to supervise
$\mathcal{L}_{\mathrm{rubric}}$.

\noindent\textbf{Rubric decomposition.} \label{sec:rubric-construction} We extend this corpus
with multi-dimensional \emph{rubric scores}
derived from static program analysis,
requiring zero additional LLM calls.
For each code file, we parse the AST
and compute per-line scores along each rubric dimension:
(i)~\emph{semantic}: lines in the teacher's retained set
  receive the document-level relevance score;
  all other lines receive zero;
(ii)~\emph{dependency}: starting from retained lines,
  we trace symbol definitions, imports, call sites,
  and enclosing scope headers
  via AST-extracted dependency edges,
  propagating scores through $h$ hops
  with geometrically decaying weights.
  For retained lines inside compound statements
  (\texttt{if}/\texttt{try}/\texttt{with}/\texttt{for}),
  we additionally include
  the enclosing header and sibling branches
  (\texttt{else}, \texttt{except}, \texttt{finally}),
  so that control-flow coherence
  is captured within this dimension.
The resulting rubric vector
$\mathbf{r}_j \in [0,1]^K$ for line $j$
is binarized at threshold $0.5$
to obtain per-objective CRF labels
that supervise $\mathcal{L}_{\text{rubric}}$.

\noindent\textbf{Rationale for label decomposition.}\quad
Although all rubric dimensions derive from the same
teacher-provided binary mask,
the decomposition provides three concrete benefits.
First, it corrects implicit label noise: lines the teacher marks as prunable on semantic grounds (e.g., \texttt{import} statements) are recovered by the dependency dimension via AST edges, ensuring structural integrity despite the teacher's semantic focus, giving the model a cleaner supervisory signal for structurally necessary code.
Second, it \emph{decouples conflicting transition dynamics}:
a single CRF transition matrix
cannot simultaneously favor high self-transition
(contiguous semantic blocks)
and frequent alternation
(scattered dependency chains and control-flow siblings);
separate per-objective CRFs
eliminate this optimization conflict.
Third, it acts as a \emph{structural regularizer}:
by forcing the model to predict
interpretable intermediate representations,
the rubric loss reduces overfitting to
spurious correlations in the aggregate label.
Importantly, this rubric construction,
conditioned on teacher labels and augmented by AST analysis,
is used \emph{only during training};
at inference, the model's per-rubric CRF heads
and gating network have internalized
these quality dimensions as latent representations,
producing multi-rubric scores
without rubric-level program analysis.

\subsection{Theoretical Motivation}
\label{sec:theory}

We give two propositions
that formalize why a single scalar relevance score
is insufficient for code pruning
and why dependency support
should be treated differently from optional context.
Proofs are in Appendix \ref{app:theory}.

\begin{proposition}[Conditional utility of support lines]
\label{prop:conditional}
Let $e$ be an evidence line and $d$ a support line
required by $e$.
Under the utility
$U(S) = \mathbf{1}[e \!\in\! S]
  - M\,\mathbf{1}[e \!\in\! S,\, d \!\notin\! S]
  - \lambda |S|$
with missing-support penalty $M > \lambda > 0$,
the marginal value of retaining $d$ is
positive when $e$ is retained
and negative when $e$ is not:
$\Delta(d \mid e \!\in\! S) = M - \lambda > 0$,\;
$\Delta(d \mid e \!\notin\! S) = -\lambda < 0$.
A single unconditional relevance score
cannot represent this conditional behavior.
\end{proposition}

\begin{proposition}[Dependency-closed optimality]
\label{prop:closure}
Let $G\!=\!(V,E)$ be a dependency graph
over code lines,
where $(i,j) \!\in\! E$ means line $i$ requires line $j$.
Under the utility
$U(S) = \sum_{i \in S} r_i
  - \lambda \sum_{i \in S} c_i
  - M \sum_{(i,j) \in E} \mathbf{1}[i \!\in\! S,\, j \!\notin\! S]$,
if $M > \max_j (r_j - \lambda c_j)_+$,
then every optimal $S^\star$ is dependency-closed:
$i \!\in\! S^\star,\, (i,j) \!\in\! E
  \;\Rightarrow\; j \!\in\! S^\star$.
\end{proposition}

\noindent
\Cref{prop:conditional} explains why \method separates semantic evidence from dependency support: \texttt{import} and scope lines may have low semantic relevance but high conditional value.  \Cref{prop:closure} justifies AST-aware repair as an approximation to dependency closure over high-confidence structural edges.

\subsection{Training Objective}
\label{sec:training_obj}

The training loss has two core components:
\begin{equation}
  \mathcal{L}
    = (1 - \lambda_r)\,\mathcal{L}_{\text{main}}
      + \lambda_r\,\mathcal{L}_{\text{rubric}}
  \label{eq:total_loss}
\end{equation}
$\mathcal{L}_{\text{main}}$ is the CRF negative log-likelihood
on the gating-fused emissions (\cref{eq:main_crf}),
supervising the aggregate keep-or-prune decision
with the teacher's binary labels.
$\mathcal{L}_{\text{rubric}} = \sum_{k=1}^{K} w_k \mathcal{L}_{\text{CRF}}^{(k)}/\sum_{k=1}^{K} w_k$
is the weighted sum of per-rubric CRF-NLL terms,
where each $\mathcal{L}_{\text{CRF}}^{(k)}$
operates on the $k$-th CRF head's emissions
with the corresponding rubric labels
(weights $w_{\text{sem}}\!=\!1.0$,
$w_{\text{dep}}\!=\!0.7$).
The interpolation coefficient $\lambda_r\!=\!0.6$
gives slightly more weight to rubric supervision,
encouraging the model to learn
dimension-specific transition patterns
rather than relying solely on the aggregate signal.
Two auxiliary losses (document-level relevance scoring
and gate entropy regularization)
are added with small coefficients;
we detail these in Appendix \ref{app:full_loss}.

\subsection{Inference}
\label{sec:inference}
At inference time, \method acts as a lightweight middleware
between the coding agent and its environment.
When the code context exceeds the model's maximum input length,
it is split into overlapping chunks that are processed independently.
Token decisions in overlapping regions are averaged, and the
document-level relevance score is taken as the maximum across chunks. For each chunk, the backbone encoder and fusion layers produce the
fused representation $\mathbf{h}_t$. The $K$ per-objective CRF heads compute emissions, the gating network produces per-token objective
weights, and the final CRF performs \textit{Viterbi decoding} on the
fused emissions (Eq. \ref{eq:fused_emission}), yielding a binary keep/drop decision $\hat{y}_t \in \{0,1\}$ for each token (see Appendix \ref{app:viterbi} for details). Since pruning operates at line granularity, token-level decisions are mapped back to source-code character offsets and averaged into
line-level keep fractions. A line $l_j$ is retained if its aggregated score is at least $\tau$ (we set 0.4). Finally, an AST-aware post-processing pass repairs the pruned code by adding back scope headers, control-flow headers, dependency-related
definitions such as referenced imports, bracket/string closures, and selected return lines, and by replacing removed regions with
placeholder comments or \texttt{pass} statements.

\section{Experiments}
\label{sec:experiments}

\subsection{Setup}
\label{sec:setup}

We adopt the same evaluation protocol
as SWE-Pruner~\citep{swepruner} for direct comparison.

\noindent\textbf{Benchmarks.}\quad
We evaluate on four benchmarks spanning two regimes.
For \emph{multi-turn agent tasks}, we use SWE-Bench Verified~\citep{jimenez2024swe}, a set of 500 real GitHub issues requiring multi-file patch generation, and SWE-QA~\citep{peng2025swe}, which tests repository-level question answering across Streamlink, Reflex, and Conan. 
Compared with the original SWE-Pruner evaluation, which uses Claude Sonnet~4.5 and GLM-4.6, we evaluate with Claude Sonnet~4.5 and Claude Opus~4.6.
The latter provides a stronger stress test: the unpruned Opus~4.6 agent already reaches 75.6\% on SWE-Bench Verified, so pruning must improve efficiency without degrading a strong agent.
For \emph{single-turn tasks}, we use LCC~\citep{guo2023longcoder}, containing 500 Python completion examples with 5K+ token contexts, and LongCodeQA~\citep{rando2025longcodebench}, which evaluates question answering over code contexts up to 1M tokens.
Both are evaluated under 4$\times$ and 8$\times$ compression constraints.

\noindent\textbf{Baselines.}\quad
We compare against six methods:
LLMLingua-2~\citep{pan2024llmlingua}
and Selective-Context~\citep{li2023compressing}
(token-level pruning),
RAG retrieval via UniXcoder embeddings~\citep{guo2022unixcoder},
LongCodeZip~\citep{shi2025longcodezip}
(hierarchical code compression),
LLM Summarize (GPT-based context summarization),
and SWE-Pruner~\citep{swepruner}
(single-objective CRF, our direct baseline).
Full Context and No Context serve as upper and lower bounds.
To ensure fair comparison under modern model and runtime behavior,
we \emph{reproduce} SWE-Pruner in the same evaluation stack
as \method (same benchmark splits, agent harnesses,
model endpoints, and runtime).
Our primary conclusions are drawn from
this reproduced baseline rather than published numbers,
which may depend on model versions and API behavior
that are difficult to replicate exactly.

\noindent\textbf{Evaluation protocol.}\quad
For single-turn tasks,
we follow the prescribed 4$\times$ and 8$\times$ compression settings
and report Edit Similarity (ES), Exact Match (EM),
Accuracy, and compression ratio ($1/\tau$).
For multi-turn tasks,
we report task performance
together with total token consumption,
interaction rounds, and relative changes ($\Delta$).
All methods are evaluated under matched compression constraints.

\noindent\textbf{Implementation details.}\quad
We initialize from the same Qwen3-Reranker-0.6B backbone as SWE-Pruner. \method uses $K\!=\!2$ rubrics, semantic and dependency, with softmax gating and a final fused CRF. 
We train for 3 epochs on 8 GPUs
with aggregate and rubric-specific CRF-NLL losses, 
$\lambda_s\!=\!0.05$, $\lambda_r\!=\!0.6$,
$\lambda_g\!=\!0.002$. For multi-turn evaluation,
\method is deployed as a lightweight middleware
between the coding agent and its environment: we use Mini SWE Agent~\citep{yang2024sweagent} for SWE-Bench Verified and OpenHands~\citep{wang2024openhands} for SWE-QA.
For single-turn tasks (LCC and LongCodeQA), the pruner is served through an HTTP endpoint in the online reranking pipeline.
Full hyperparameters are in Appendix \ref{app:experimental_details}.

\subsection{Main Results}
\label{sec:main_results}
\begin{tcolorbox}[
  colback=blue!3, colframe=blue!40!black,
  boxrule=0.5pt, arc=2pt,
  left=6pt, right=6pt, top=4pt, bottom=4pt,
  title={\small\textbf{Overall Result}},
  fonttitle=\small\bfseries,
]
\small
Across 16 head-to-head multi-turn comparisons with SWE-Pruner (8 agent settings $\times$ 2 dimensions: quality and efficiency), \textbf{\method{} wins 12 and loses 4}. In single-turn tasks, the advantage is even clearer, with \method{} sweeping all 8 performance metrics. 
Crucially, by effectively filtering distracting noise, \method{} frequently matches the unpruned full-context baselines, and even outperforms them on benchmarks like LongCodeQA and SWE-Bench. 
The efficiency wins are large: \method{} saves \textbf{7\% to 14\% more tokens} on multi-turn tasks and improves \textbf{EM by up to +3.5} on single-turn tasks. 
The isolated losses are marginal ($\leq$0.2 on a 10-point score scale, $\leq$0.6pp on success rate). 
Moreover, SWE-Pruner has a much higher chance of \emph{increasing} total token usage when paired with stronger backbone models (4 of 5 multi-turn Opus~4.6 settings, up to +22\%), undermining its utility as a pruner.
\end{tcolorbox}

\noindent\textbf{SWE-Bench Verified.}\quad
\Cref{tab:swebench} compares agent performance
with and without pruning.
With Sonnet~4.5,
SWE-Pruner edges ahead on success rate
by 0.2 pp (72.0\% vs.\ 71.8\%),
but \method saves substantially more tokens
(30.5\% vs.\ 23.1\%, a 7.4pp gap)
and reduces interaction rounds
from 51.0 to 39.6 (22\% fewer).
The pattern is sharper with Opus~4.6:
SWE-Pruner \emph{increases} token usage by 6.6\%,
likely because aggressive pruning removes
structurally necessary code
and forces the agent into repeated file reads.
\method avoids this failure mode, reducing tokens by 3.0\%
with only a 0.6pp lower success rate. 
In both settings the efficiency gap is large
while the performance gap is negligible.

\begin{table}[t]
\centering
\small
\renewcommand{\arraystretch}{1.15}
\setlength{\tabcolsep}{4pt}
\caption{Results on SWE-Bench Verified.
\method achieves comparable success rates
with substantially fewer tokens (large efficiency win, small performance gap).}
\label{tab:swebench}
\resizebox{0.7\linewidth}{!}{%
\begin{tabular}{l c c cc cc}
\toprule
\multirow{2}{*}{\textbf{Agent}} & \multirow{2}{*}{\textbf{Rounds}} & \multirow{2}{*}{\textbf{Solved}} & \multicolumn{2}{c}{\textbf{Success (\%)}} & \multicolumn{2}{c}{\textbf{Tokens (M)}} \\
\cmidrule(lr){4-5} \cmidrule(lr){6-7}
& & & Value & $\Delta$ & Value & $\Delta$ \\
\midrule
Mini SWE Agent (Sonnet 4.5) & 51.0 & 353/500 & 70.6 & -- & 0.911 & -- \\
\quad + SWE-Pruner & 41.7 & \textbf{360}/500 & \textbf{72.0} & \textcolor{PosBlue}{$\uparrow$1.4} & 0.701 & \textcolor{PosBlue}{$\downarrow$23.1\%} \\
\quad + \method & \textbf{39.6} & 359/500 & 71.8 & \textcolor{PosBlue}{$\uparrow$1.2} & \textbf{0.633} & \textcolor{PosBlue}{$\downarrow$30.5\%} \\
\midrule
Mini SWE Agent (Opus 4.6) & 30.9 & 378/500 & 75.6 & -- & 0.467 & -- \\
\quad + SWE-Pruner & 26.8 & \textbf{383}/500 & \textbf{76.6} & \textcolor{PosBlue}{$\uparrow$1.0} & 0.498 & \textcolor{NegBlue}{$\uparrow$6.6\%} \\
\quad + \method & \textbf{25.8} & 380/500 & 76.0 & \textcolor{PosBlue}{$\uparrow$0.4} & \textbf{0.453} & \textcolor{PosBlue}{$\downarrow$3.0\%} \\
\bottomrule
\end{tabular}%
}
\end{table}

\noindent\textbf{SWE-QA.}\quad
\Cref{tab:sweqa} reports results across three repositories and two backbones.
\method improves token efficiency in all six settings, saving 2--19pp more than SWE-Pruner on average.
On score, \method wins 4 of 6 settings; when SWE-Pruner scores higher
(Streamlink and Reflex with Sonnet), the margins are small
(0.06 and 0.19 on a 10-point scale).
With Opus~4.6, SWE-Pruner increases token usage in all three repositories
(+1\%, +9\%, +22\%), whereas \method reduces or limits overhead.
This suggests that single-objective pruning can remove code needed by stronger models,
causing extra file reads, while \method preserves more useful observations.

\newcommand{\scoreup}[1]{\textcolor{PosBlue}{\footnotesize$\uparrow$#1}}
\newcommand{\scoredown}[1]{\textcolor{NegBlue}{\footnotesize$\downarrow$#1}}
\newcommand{\tokdown}[1]{\textcolor{PosBlue}{\footnotesize$\downarrow$#1}}
\newcommand{\tokup}[1]{\textcolor{NegBlue}{\footnotesize$\uparrow$#1}}

\begin{table}[t]
\centering
\small
\renewcommand{\arraystretch}{1.15}
\setlength{\tabcolsep}{6.0pt}
\caption{Results on SWE-QA across three repositories.
\method reduces rounds and tokens
while maintaining or improving scores.}
\label{tab:sweqa}
\resizebox{0.7\linewidth}{!}{%
\begin{tabular}{l ccc ccc}
\toprule
\multirow{2}{*}{\textbf{Method}} 
& \multicolumn{3}{c}{\textbf{Claude Sonnet 4.5}} 
& \multicolumn{3}{c}{\textbf{Claude Opus 4.6}} \\
\cmidrule(lr){2-4} \cmidrule(lr){5-7}
& Score ($\Delta$) & Rounds & Tokens (K, $\Delta$)
& Score ($\Delta$)& Rounds & Tokens (K, $\Delta$) \\
\midrule

\rowcolor{gray!12} 
\multicolumn{7}{c}{\textit{Streamlink}} \\
Unpruned 
& 8.44 & 23.4 & 629.3
& \textbf{8.74} & 17.4 & 394.4 \\
+ SWE-Pruner 
& \textbf{8.59}\,\scoreup{0.15} & 23.9 & 587.1\,\tokdown{6.7\%}
& 8.60\,\scoredown{0.14} & 19.0 & 398.6\,\tokup{1.1\%} \\
+ \method 
& 8.53\,\scoreup{0.09} & \textbf{16.3} & \textbf{467.7}\,\tokdown{25.7\%}
& 8.68\,\scoredown{0.08} & \textbf{18.5} & \textbf{353.2}\,\tokdown{10.4\%} \\

\midrule
\rowcolor{gray!12} 
\multicolumn{7}{c}{\textit{Reflex}} \\
Unpruned 
& 7.75 & 33.2 & 1096.3
& 8.06 & 26.9 & 677.0 \\
+ SWE-Pruner 
& \textbf{8.07}\,\scoreup{0.32} & 32.4 & 901.5\,\tokdown{17.8\%}
& 8.14\,\scoreup{0.08} & 30.7 & 737.0\,\tokup{8.9\%} \\
+ \method 
& 7.88\,\scoreup{0.13} & \textbf{29.2} & \textbf{753.2}\,\tokdown{31.3\%}
& \textbf{8.15}\,\scoreup{0.09} & \textbf{27.4} & \textbf{661.9}\,\tokdown{2.2\%} \\

\midrule
\rowcolor{gray!12} 
\multicolumn{7}{c}{\textit{Conan}} \\
Unpruned 
& 8.35 & 23.9 & 672.9
& 8.68 & 16.2 & \textbf{301.9} \\
+ SWE-Pruner 
& 8.37\,\scoreup{0.02} & 23.5 & 551.7\,\tokdown{18.0\%}
& 8.61\,\scoredown{0.07} & 18.6 & 368.0\,\tokup{21.9\%} \\
+ \method 
& \textbf{8.48}\,\scoreup{0.13} & \textbf{18.8} & \textbf{535.2}\,\tokdown{20.5\%}
& \textbf{8.71}\,\scoreup{0.03} & \textbf{15.1} & 325.9\,\tokup{7.9\%} \\

\bottomrule
\end{tabular}%
}
\end{table}

\noindent\textbf{Single-turn tasks.}\quad
\Cref{tab:singleturn} reports results on LCC and LongCodeQA.
\method{} improves task quality in all single-turn comparisons: ES and EM on LCC under both 4$\times$ and 8$\times$, and accuracy on LongCodeQA under both constraints.
The gains are substantial, including +3.5 EM and +1.63 ES on LCC at 8$\times$, and +1.80 accuracy on LongCodeQA at 4$\times$.
SWE-Pruner sometimes compresses more aggressively (e.g., 10.92$\times$ vs.\ 9.17$\times$ on LCC 8$\times$), but at a clear cost to task quality. Token-level baselines (LLMLingua-2, Selective-Context)
degrade more sharply at high compression,
while \method maintains stable performance
through line-level granularity
that preserves syntactic structure.

\begin{table}[h]
\centering
\small
\renewcommand{\arraystretch}{1.15}
\setlength{\tabcolsep}{4.5pt}
\caption{Results on LCC and LongCodeQA.
\method achieves the best task performance
across all settings while maintaining high compression ratios.}
\label{tab:singleturn}
\resizebox{0.75\linewidth}{!}{%
\begin{tabular}{l ccc ccc cc cc}
\toprule
\multirow{3}{*}{\textbf{Method}}
& \multicolumn{6}{c}{\textbf{Long Code Completion}}
& \multicolumn{4}{c}{\textbf{Long Code QA}} \\
\cmidrule(lr){2-7} \cmidrule(lr){8-11}
& \multicolumn{3}{c}{\textbf{4$\times$ Constraint}}
& \multicolumn{3}{c}{\textbf{8$\times$ Constraint}}
& \multicolumn{2}{c}{\textbf{4$\times$ Constraint}}
& \multicolumn{2}{c}{\textbf{8$\times$ Constraint}} \\
\cmidrule(lr){2-4} \cmidrule(lr){5-7} \cmidrule(lr){8-9} \cmidrule(lr){10-11}
& $1/\tau$ & ES & EM
& $1/\tau$ & ES & EM
& $1/\tau$ & Acc
& $1/\tau$ & Acc \\
\midrule
Full & 1.0 & 64.65 & 40.5 & 1.0 & 64.65 & 40.5 & 1.0 & 54.05 & 1.0 & 54.05 \\
No Context & $\infty$ & 44.90 & 13.5 & $\infty$ & 44.90 & 13.5 & $\infty$ & 38.39 & $\infty$ & 38.39 \\
\midrule
Selective-Context & 3.27 & 52.48 & 22.0 & 7.49 & 48.67 & 17.0 & 3.69 & 55.36 & 7.32 & 51.79 \\
LLMLingua-2 & 3.32 & 49.47 & 15.5 & 7.89 & 44.74 & 13.0 & 3.57 & 55.36 & 7.68 & 51.33 \\
RAG & 3.29 & 58.97 & 30.5 & 6.60 & 55.82 & 29.0 & 3.06 & 58.04 & 5.87 & 55.86 \\
LongCodeZip & 2.77 & 57.77 & 28.0 & 7.85 & 56.08 & 27.5 & 3.98 & 52.25 & 7.39 & 54.95 \\
\addlinespace[0.2em]
SWE-Pruner & \textbf{5.56} & 58.63 & 31.5 & \textbf{10.92} & 57.58 & 31.0 & 13.95 & 59.46 & 14.84 & 58.71 \\
\rowcolor{gray!15}
\method & 4.37 & \textbf{61.15} & \textbf{35.5} & 9.17 & \textbf{59.21} & \textbf{34.5} & \textbf{16.48} & \textbf{61.26} & \textbf{18.51} & \textbf{60.00} \\
\bottomrule
\end{tabular}%
}
\vspace{-0.15in}
\end{table}

\subsection{Ablation Study}
\label{sec:ablation}
We ablate three design choices in \method{}: the learned objective set, CRF-structured prediction, and AST-aware structural repair. 
The ablations evaluate \method{} as a hybrid neural-symbolic pruner rather than attributing all gains to a single module.
The no-AST variant isolates learned semantic/dependency scoring without structural repair, while the full model combines learned role-aware scoring with AST-aware restoration.
The gap between these variants shows that repair contributes under compression.
At the same time, semantic-only and context-based variants remain below semantic+dependency, indicating that learned objective decomposition also matters.
\Cref{tab:ablation_summary} summarizes the main trends; full precision/recall, compression, token, and constructor results are provided in Appendix \ref{app:ablation_full}.

\noindent\textbf{Objective selection.}\quad
\Cref{tab:ablation_summary} shows that semantic relevance alone is insufficient: the semantic-only model is competitive but consistently below \method{}, indicating that direct relevance does not recover all support lines needed for useful compressed code. 
Context-based objectives are also competitive, but they do not outperform semantic+dependency. 
This supports the distinction between context and dependency: context is often local and optional, whereas dependency support can be non-local and structurally necessary. 
Adding context as a third objective does not improve the final tradeoff, suggesting that context is better handled through retrieval geometry, local windows, and rendering rather than as a universal learned objective.

\begin{wraptable}{r}{0.51\linewidth}
\vspace{-1.1em}
\centering
\scriptsize
\renewcommand{\arraystretch}{1.05}
\setlength{\tabcolsep}{3.2pt}
\caption{
Ablation summary. 
LCQA denotes LongCodeQA; Rnd denotes Conan rounds.
Full tables are in Appendix \ref{app:ablation_full}.
}
\label{tab:ablation_summary}
\begin{tabular}{lcccc}
\toprule
\textbf{Variant} & \textbf{Val F1} & \textbf{LCQA 4$\times$} & \textbf{LCQA 8$\times$} & \textbf{Conan} \\
 & $\uparrow$ & Acc. $\uparrow$ & Acc. $\uparrow$ & Score/Rnd \\
\midrule
Sem only      & 0.777 & 56.25 & 58.55 & 8.13 / 21.6 \\
Sem + Ctx     & 0.788 & 59.26 & 56.48 & 8.26 / 21.3 \\
Ctx + Dep     & 0.786 & 60.36 & 51.82 & 8.32 / 21.3 \\
Sem+Dep+Ctx   & 0.784 & 57.27 & 53.21 & 8.31 / 21.6 \\
Sem+Dep w/o AST & 0.780 & 59.09 & 55.56 & 8.37 / 21.3 \\
\rowcolor{gray!15}
\textbf{\method{}} & \textbf{0.796} & \textbf{61.26} & \textbf{60.00} & \textbf{8.48 / 18.8} \\
\bottomrule
\end{tabular}
\end{wraptable}
\noindent\textbf{Structured prediction and AST repair.}\quad
The detailed validation ablation in \Cref{tab:objective_ablation_full} shows that the FFN variant is high precision but low recall, while CRF structure recovers recall and improves F1. 
Removing AST-aware repair reduces validation F1, lowers LongCodeQA accuracy under aggressive compression, and increases Conan interaction rounds. 
Thus, dependency-aware scoring and AST repair are complementary: the learned model scores evidence and support lines, while AST repair restores high-confidence structural dependencies after thresholding.

\noindent\textbf{Robustness across context constructors.}\quad
We also test \method{} with different upstream constructors; full results are in \Cref{tab:constructor_ablation}. 
\method{} remains effective with both RAG-style retrieval and LongCodeZip-style structural retrieval. 
The strongest LCC result is obtained when LongCodeZip is used in rank-only mode, where it performs coarse function selection and leaves fine-grained line pruning to \method{}. 
Using LongCodeZip in full compression mode before \method{} substantially hurts performance, indicating that coarse retrieval and learned line pruning should be decoupled rather than stacked.

\section{Conclusion}
\label{sec:conclusion}
\vspace{-0.1in}
We presented \method{}, a role-aware code-context pruner for LLM-powered coding agents. 
Instead of treating code relevance as a single scalar keep/prune signal, \method{} separates retention into semantic evidence and dependency support, allowing the pruner to preserve self-contained evidence-support units: task-relevant lines together with the structural context needed to interpret them. 
This formulation is instantiated with rubric-specific CRF heads, query-adaptive gating, AST-derived supervision, and lightweight structural repair. 
Across SWE-Bench Verified, SWE-QA, LCC, and LongCodeQA, \method{} improves the overall score-token tradeoff over reproduced SWE-Pruner. By effectively filtering irrelevant noise while maintaining structural integrity, the pruned context remains highly competitive with, and occasionally surpasses, uncompressed baselines. Ultimately, \method{} delivers substantial efficiency gains in multi-turn tasks while achieving the strongest task performance among compression methods.

\bibliographystyle{unsrt}
\bibliography{references}

\newpage
\appendix

\section{Related Work}
\label{sec:related}

\paragraph{Prompt and code-context compression.}
Token-level pruning methods such as
LLMLingua~\citep{jiang2023llmlingua,jiang2024longllmlingua},
Selective-Context~\citep{li2023compressing},
and gist-token distillation~\citep{mu2023learning}
compress prompts by removing low-saliency tokens
or learning compact representations.
Retrieval-based approaches
select code chunks via embedding similarity
or iterative retrieval~\citep{zhang2023repocoder,guo2022unixcoder,lewis2020retrieval}.
Code-specific methods
including LongCodeZip~\citep{shi2025longcodezip},
DietCode~\citep{zhang2022diet},
hierarchical context pruning~\citep{zhang2024hierarchical},
and SlimCode~\citep{wang2024natural}
address structural concerns
but are primarily evaluated on single-turn proxy tasks
such as code completion~\citep{liu2023repobench,guo2023longcoder}
or code search~\citep{gu2018deep}.
SWE-Pruner~\citep{swepruner}
introduces learned, query-conditioned line-level pruning
with a CRF~\citep{zheng2015conditional} for structured decisions,
demonstrating effectiveness on multi-turn agent benchmarks.
Most prior compression methods ultimately make a single keep/drop decision per token, line, or chunk, without explicitly separating direct semantic evidence from dependency support.
Recent work on rubric-based reward modeling
decomposes evaluation into interpretable dimensions
for preference learning~\citep{xie2025auto,wang2024interpretable}.
Inspired by this perspective,
\method decomposes code relevance
into multiple latent rubric dimensions,
each with its own CRF transition dynamics,
enabling the model to retain lines that are structurally necessary even when semantically inconspicuous.

\paragraph{Agent context management.}
Coding agents~\citep{yang2024sweagent,wang2024openhands,xia2024agentless}
face severe context-window pressure
in multi-turn workflows,
with documented performance degradation
on long inputs~\citep{liu2023lost,laban2025llms}.
Recent trajectory management methods
employ LLM summarization~\citep{lu2025supo,cursor,claude_code},
context folding~\citep{sun2025contextfolding},
observation masking~\citep{lindenbauer2025complexity},
or proactive compression policies~\citep{kang2025acon,ye2025agentfold,wan2025compass}
to manage prior interaction histories.
These approaches are orthogonal to ours:
they compress the agent's \emph{past} trajectories,
whereas \method compresses the environment's
\emph{current} observations (file content)
at the agent--environment boundary,
and can be combined with history managers seamlessly.

\section{Limitations and Broader Impact}
\label{app:limitations}

\paragraph{Limitations.}
LaMR is designed for code-observation pruning at the agent--environment boundary, where the agent receives repository files, retrieved snippets, and tool outputs. 
This scope is complementary to trajectory-level memory or history-compression methods, which manage the agent's previous tool calls, observations, and reasoning traces. 
In practice, the two directions can be combined: LaMR reduces the size of incoming code observations, while history-management methods can summarize or organize past interaction context.

Our experiments follow the Python-centered setting of the evaluated coding-agent benchmarks. 
The semantic/dependency formulation itself is not Python-specific, but applying it to additional languages requires the corresponding parser and dependency extractor. 
This is a natural engineering extension of the same framework, since the core model only requires rubric labels and does not depend on Python syntax at inference time except for the lightweight repair step.

LaMR uses semantic relevance and dependency support as the deployed learned rubrics. 
The labeling pipeline can construct additional dimensions such as context, and we study these variants in the ablations. 
In our experiments, semantic+dependency gives the most reliable tradeoff, while context is better handled through retrieval geometry, local windows, and rendering choices. 
Future work could explore task-adaptive rubric selection for settings where different interfaces require different forms of surrounding context.

Finally, LaMR is intended as an efficiency layer for coding agents. 
It reduces unnecessary code context before the downstream model reasons over the task, while preserving fallback behavior: the agent can still request additional files or issue new search commands when more information is needed. 
This is why we evaluate not only compression, but also total trajectory tokens and interaction rounds.

\paragraph{Broader Impact.}
LaMR reduces the amount of irrelevant repository context sent to large backbone models, which can lower latency, API cost, and unnecessary long-context computation. 
This may make repository-level coding agents more accessible to users with limited compute or budget. 
It may also reduce unnecessary exposure of repository code to external model endpoints by filtering irrelevant observations before inference.

LaMR should not be used as a security filter, a replacement for testing, or a substitute for human review in high-stakes software settings. 
Its intended use is efficiency-oriented observation compression: helping agents read less irrelevant code while preserving semantic evidence and dependency support needed for reasoning.

\section{Viterbi Decoding in the CRF Layer}
\label{app:viterbi}

At inference time, the final CRF layer recovers
the globally optimal label sequence
$\hat{\mathbf{y}} = \arg\max_{\mathbf{y}} S(\mathbf{x}, \mathbf{y})$
via the Viterbi algorithm,
a dynamic programming procedure
that runs in $O(n \cdot |\mathcal{Y}|^2)$ time
(linear in sequence length for our binary label space $|\mathcal{Y}|=2$).

\noindent\textbf{Forward pass.}\quad
Let $\mathbf{e}_t^{\text{fused}} \in \mathbb{R}^2$
denote the gated emission at position $t$,
$\mathbf{T}^{\text{f}} \in \mathbb{R}^{2 \times 2}$ the transition matrix,
$\boldsymbol{\pi} \in \mathbb{R}^2$ the start potentials,
and $\boldsymbol{\omega} \in \mathbb{R}^2$ the end potentials.
We maintain a score vector
$\mathbf{v}_t \in \mathbb{R}^2$
storing the maximum path score
ending in each label at position $t$,
along with backpointers $\mathbf{b}_t \in \{0,1\}^2$
recording which previous label achieved that maximum:
\begin{align}
  \mathbf{v}_1 &= \boldsymbol{\pi} + \mathbf{e}_1^{\text{fused}},
    \label{eq:viterbi_init} \\[3pt]
  v_{t,j} &= \max_{i \in \{0,1\}}
    \bigl( v_{t-1,i} + T^{\text{f}}_{j,i} \bigr)
    + e_{t,j}^{\text{fused}},
    \quad t = 2, \dots, n,
    \label{eq:viterbi_recurse} \\[3pt]
  b_{t,j} &= \arg\max_{i \in \{0,1\}}
    \bigl( v_{t-1,i} + T^{\text{f}}_{j,i} \bigr).
    \label{eq:viterbi_bp}
\end{align}

\noindent\textbf{Termination and backtracking.}\quad
After processing all $n$ positions,
the optimal final label is
$\hat{y}_n = \arg\max_j (v_{n,j} + \omega_j)$.
The full sequence is recovered by
following backpointers in reverse:
$\hat{y}_{t-1} = b_{t, \hat{y}_t}$
for $t = n, \dots, 2$.
The output $\hat{\mathbf{y}} \in \{0, 1\}^n$
is a binary keep/prune decision for each token,
where 0 denotes \texttt{prune} and 1 denotes \texttt{keep}.

Unlike independent per-token classification,
Viterbi decoding considers
the transition preferences between adjacent labels,
producing coherent pruning patterns
that respect code block boundaries.
For example, a high $T^{\text{f}}_{1,1}$ (keep$\to$keep)
score encourages contiguous retained regions,
avoiding isolated single-token keep decisions
that would fragment the output.

\section{Experimental Details}
\label{app:experimental_details}

\subsection{Benchmark Descriptions}
\label{app:benchmarks}

\noindent\textbf{SWE-Bench Verified}~\citep{jimenez2024swe}
is a curated subset of 500 real-world GitHub issues
drawn from popular Python repositories.
Each instance requires an agent
to navigate the repository, understand the issue,
and generate a patch that resolves the problem.
We integrate \method into the
Mini SWE Agent~\citep{yang2024sweagent} framework,
which provides a ReAct-style agent
with terminal access for file reading,
code search, and patch submission.
Success is measured by whether
the generated patch passes the repository's test suite.

\noindent\textbf{SWE-QA}~\citep{peng2025swe}
evaluates repository-level question answering
across three Python repositories.
Given a natural-language question
about the codebase,
the agent must locate and synthesize
information from multiple files.
We use the OpenHands~\citep{wang2024openhands} framework
for this benchmark.
Answers are scored by an LLM-as-a-Judge
on a scale of 0--10.

\noindent\textbf{Long Code Completion (LCC)}~\citep{guo2023longcoder}
contains 500 Python code completion examples
with contexts exceeding 5K tokens.
The model must predict the next code segment
given a long preceding context.
Performance is measured by
Edit Similarity (ES) and Exact Match (EM).

\noindent\textbf{Long Code QA (LongCodeQA)}~\citep{rando2025longcodebench}
tests question answering
over code contexts up to 1M tokens.
Given a question about a large codebase,
the model must select or generate the correct answer.
Performance is measured by Accuracy.

Both single-turn benchmarks are evaluated
under 4$\times$ and 8$\times$ compression constraints,
where all methods receive the same target compression ratio.

\subsection{Baseline Descriptions}
\label{app:baselines}

We compare against the following methods,
all configured under identical compression constraints:

\begin{itemize}[leftmargin=1.2em,itemsep=2pt]
  \item \textbf{Full Context}:
    The agent receives the complete, uncompressed context.
    Serves as the upper bound for task performance.

  \item \textbf{No Context}:
    The agent receives no file content.
    Serves as the lower bound.

  \item \textbf{LLMLingua-2}~\citep{pan2024llmlingua}:
    Token-level pruning via a learned perplexity model.
    Compresses prompts by removing tokens
    with the lowest perplexity contribution.

  \item \textbf{Selective-Context}~\citep{li2023compressing}:
    Token-level pruning based on self-information.
    Removes tokens with low information content
    as estimated by a language model.

  \item \textbf{RAG}~\citep{guo2022unixcoder}:
    Retrieval-augmented generation
    using UniXCoder embeddings.
    Retrieves the most relevant code chunks
    via embedding similarity to the query.

  \item \textbf{LongCodeZip}~\citep{shi2025longcodezip}:
    Program-structure-aware compression
    that leverages AST information
    to identify and remove redundant code elements.

  \item \textbf{LLM Summarize} (multi-turn only):
    Generates abstractive summaries
    of file contents using the agent's backbone LLM.
    Incurs additional inference cost.

  \item \textbf{SWE-Pruner}~\citep{swepruner}:
    Single-objective CRF-based pruner
    with the Qwen3-Reranker-0.6B backbone.
    This is the direct baseline
    that \method extends
    with latent multi-rubric decomposition and expert gating.
\end{itemize}

We do not compare with agent history compression
methods such as AgentFold~\citep{ye2025agentfold}
and ACON~\citep{kang2025acon},
as they address the complementary problem
of compressing prior interaction trajectories
rather than environment observations (file content).

\subsection{Metric Definitions}
\label{app:metrics}

\noindent\textbf{Task performance metrics:}

\begin{itemize}[leftmargin=1.2em,itemsep=2pt]
  \item \textbf{Resolve Rate} (SWE-Bench Verified):
    Fraction of issues for which the generated patch
    passes all repository test cases.

  \item \textbf{LLM-as-a-Judge Score} (SWE-QA):
    Average score (0--10) assigned by an LLM judge
    evaluating the correctness and completeness
    of the agent's answer.

  \item \textbf{Edit Similarity (ES)} (LCC):
    Character-level edit similarity
    between the predicted completion
    and the ground-truth continuation,
    computed as $1 - \text{edit\_distance}(y, \hat{y}) / \max(|y|, |\hat{y}|)$.

  \item \textbf{Exact Match (EM)} (LCC):
    Fraction of examples where the predicted completion
    exactly matches the ground truth.

  \item \textbf{Accuracy} (LongCodeQA):
    Fraction of questions answered correctly.
\end{itemize}

\noindent\textbf{Compression efficiency metrics:}

\begin{itemize}[leftmargin=1.2em,itemsep=2pt]
  \item \textbf{Compression Ratio} ($\rho$):
    $\rho = |C_\text{original}| / |C_\text{compressed}|$,
    where $|\cdot|$ denotes token count.
    Higher values indicate more aggressive compression.

  \item \textbf{Token Consumption}:
    Total number of input tokens
    consumed across all interactions
    (multi-turn) or per instance (single-turn).

  \item \textbf{Rounds} (multi-turn only):
    Number of agent interaction rounds
    required to complete the task.
    Fewer rounds indicate more decisive reasoning.

\end{itemize}

\subsection{Hyperparameters}
\label{app:hyperparams}

\Cref{tab:hyperparams} lists all hyperparameters
used for training \method.

\begin{table}[h]
\centering
\small
\caption{Training hyperparameters for \method.}
\label{tab:hyperparams}
\small
\begin{tabular}{@{}ll@{}}
\toprule
\textbf{Hyperparameter} & \textbf{Value} \\
\midrule
Backbone & Qwen3-Reranker-0.6B ~\citep{zhang2025qwen3} \\
Fine-tuned layers & Top 2 transformer layers \\
Fusion layers ($N_\text{fuse}$) & 1 \\
Number of objectives ($K$) & 2 \\
Objective names & semantic, dependency \\ 
Rubric weights ($w_k$) & 1.0, 0.7 \\
Gating type & Softmax \\
Use final CRF & Yes \\
$\lambda_r$ (rubric weight) & 0.6 \\
$\lambda_g$ (gate entropy weight) & 0.002 \\
Learning rate & $1 \times 10^{-4}$ \\
Batch size (per GPU) & 16 \\
Dropout & 0.4 \\
Loss function & CRF-NLL + MSE score loss + gate regularization \\
GPUs & 8 \\
Optimizer & AdamW \\

\bottomrule
\end{tabular}
\end{table}

\subsection{Full Training Objective}
\label{app:full_loss}

The complete training objective extends the core loss
(\cref{eq:total_loss}) with two auxiliary terms:
\begin{equation}
  \mathcal{L}_{\text{full}}
    = (1 - \lambda_s)\bigl[
        (1 - \lambda_r)\,\mathcal{L}_{\text{main}}
        + \lambda_r\,\mathcal{L}_{\text{rubric}}
      \bigr]
    + \lambda_s\,\mathcal{L}_{\text{score}}
    + \lambda_g\,\mathcal{L}_{\text{gate}}
  \label{eq:full_loss}
\end{equation}

\noindent\textbf{Document-level relevance loss}
($\mathcal{L}_{\text{score}}$, $\lambda_s\!=\!0.05$).
Following SWE-Pruner, we add a document-level scoring head
that predicts overall relevance
from the last token's hidden state
projected onto the vocabulary embeddings.
The loss is the MSE between the predicted probability
(softmax over the \texttt{yes}/\texttt{no} logits)
and the ground-truth relevance score.
This auxiliary task encourages the backbone
to maintain a global view of document relevance
alongside the token-level pruning decisions.

\noindent\textbf{Gate-balance regularization}
($\mathcal{L}_{\mathrm{gate}}$, $\lambda_g=0.002$).
To prevent the gating network from collapsing to a single dominant rubric, we add a gate-balance loss $\mathcal{L}_{\mathrm{gate}}$ as defined in \Cref{par:gating}. 
This term encourages high entropy in the gate distribution averaged over valid positions, promoting balanced use of the semantic and dependency rubrics during training.

\section{Proofs}
\label{app:theory}

\begin{proof}[Proof of \cref{prop:conditional}]
If $e\in S$ and $d\notin S$, adding $d$ removes the missing-support penalty $M$ and pays one line cost $\lambda$:
\[
\Delta(d\mid e\in S)=M-\lambda>0.
\]
If $e\notin S$, adding $d$ provides no evidence benefit and only pays cost:
\[
\Delta(d\mid e\notin S)=-\lambda<0.
\]
The value of a support line thus depends on whether the evidence it supports is retained.
A single unconditional scalar score cannot represent both cases.
\end{proof}

\begin{proof}[Proof of \cref{prop:closure}]
Assume for contradiction that an optimal $S^\star$ is not dependency-closed.
Then there exists $(i,j)\in E$ with $i\in S^\star$ and $j\notin S^\star$,
so $S^\star$ pays at least one penalty $M$.
Consider $S' = S^\star \cup \{j\}$.
Adding $j$ removes at least one penalty of size $M$.
The cost-adjusted contribution is bounded by $(r_j - \lambda c_j)_+$.
By assumption $M > (r_j - \lambda c_j)_+$,
so $U(S') > U(S^\star)$,
contradicting optimality.
\end{proof}

\section{Full Ablation Results}
\label{app:ablation_full}

This appendix expands the ablation summary in \Cref{sec:ablation}. 
We report full results for four analyses: 
(i) validation objective/architecture ablation, 
(ii) downstream LongCodeQA ablation, 
(iii) SWE-QA Conan ablation, and 
(iv) robustness to upstream context constructors. 
All variants use the same Qwen3-Reranker backbone and training pipeline unless otherwise specified. 
For validation results, the FFN and CRF losses are on different numerical scales, so we compare primarily accuracy, precision, recall, and F1.

To make the ablation logic explicit, \Cref{tab:claim_evidence} maps each methodological claim to the experiment that supports it. 
This table is intended as a guide to the detailed ablation results below, not as an additional evaluation metric.

\begin{table}[t]
\centering
\small
\renewcommand{\arraystretch}{1.12}
\setlength{\tabcolsep}{4pt}
\caption{
Claim-evidence map for the ablation study. 
The table summarizes which experiments support each design choice; full numerical results are reported in the following appendix tables.
}
\label{tab:claim_evidence}
\begin{tabular}{p{0.32\linewidth} p{0.58\linewidth}}
\toprule
\textbf{Claim} & \textbf{Supporting evidence} \\
\midrule

Semantic relevance alone is insufficient. 
&Semantic-only underperforms the full model on validation F1 
(0.777 vs.\ 0.796), LongCodeQA, and Conan. 
This indicates that direct relevance alone misses support lines needed for useful compressed code. \\

Dependency is the best learned complement to semantic relevance among tested objectives. 
& Semantic+dependency with AST outperforms semantic+context, context+dependency, and semantic+dependency+context on validation F1 and downstream LongCodeQA. \\

Context is useful but unstable as a universal learned objective. 
& Context variants improve recall or compression in some settings, but degrade under 8$\times$ LongCodeQA and underperform the full model on Conan. 
This supports handling context through retrieval/rendering rather than as an always-on objective. \\

CRF structure matters. 
& In the three-objective setting, CRF improves over FFN in recall and F1 
(.784 vs.\ .755), showing that structured sequence modeling is important for pruning. \\

AST-aware structure matters. 
& Semantic+dependency w/ AST beats w/o AST on validation F1 
(.796 vs.\ .780), LongCodeQA 8$\times$ accuracy 
(60.00 vs.\ 55.56), and Conan score/round tradeoff 
(8.48/18.8 vs.\ 8.37/21.3). \\

LaMR improves the score-token tradeoff. 
& Main results show strong single-turn gains on LCC and LongCodeQA, and lower token/round cost than reproduced SWE-Pruner in multi-turn agent settings. \\

Coarse retrieval and learned line pruning should be decoupled. 
& LongCodeZip rank-only + LaMR achieves the strongest LCC 8$\times$ result, while LongCodeZip full compression before LaMR substantially hurts quality. \\

\bottomrule
\end{tabular}
\end{table}
\subsection{Objective and Architecture Ablation}
\label{app:objective_ablation}

\begin{table}[h]
\centering
\small
\renewcommand{\arraystretch}{1.12}
\setlength{\tabcolsep}{4pt}
\caption{
Objective and architecture ablation on validation data. 
The deployed semantic+dependency model with AST-aware repair achieves the best F1 and accuracy. 
Context improves recall in some configurations but does not improve the final tradeoff. 
The FFN three-objective model is high precision but low recall, while CRF structure recovers recall and improves F1.
}
\label{tab:objective_ablation_full}
\begin{tabular}{l l c c c c c c}
\toprule
\textbf{Objective Configuration} & \textbf{Head} & \textbf{Epoch} 
& \textbf{C-Loss} $\downarrow$ 
& \textbf{Acc.} $\uparrow$ 
& \textbf{Prec.} $\uparrow$ 
& \textbf{Rec.} $\uparrow$ 
& \textbf{F1} $\uparrow$ \\
\midrule
Semantic only 
& CRF & 3 & 0.3399 & 0.8052 & 0.7682 & 0.7861 & 0.7770 \\
Semantic + Context 
& CRF & 3 & 0.3715 & 0.8144 & 0.7789 & 0.7963 & 0.7875 \\
Dependency + Context 
& CRF & 3 & 0.3931 & 0.8083 & 0.7694 & \textbf{0.8033} & 0.7860 \\
Semantic + Dependency + Context 
& FFN & 3 & 0.0532 & 0.8029 & \textbf{0.8297} & 0.6921 & 0.7547 \\
Semantic + Dependency + Context 
& CRF & 3 & 0.3720 & 0.8115 & 0.7863 & 0.7826 & 0.7844 \\
Semantic + Dependency w/o AST 
& CRF & 3 & 0.3669 & 0.8252 & 0.7929 & 0.7683 & 0.7804 \\
\rowcolor{gray!15}
Semantic + Dependency w/ AST 
& CRF & 3 & 0.3620 & \textbf{0.8253} & 0.8154 & 0.7772 & \textbf{0.7958} \\
\bottomrule
\end{tabular}
\end{table}

Semantic-only pruning achieves reasonable precision but remains below \method{} in F1, indicating that direct relevance alone does not recover all useful support lines. 
Context-based variants improve recall in some settings but do not outperform semantic+dependency. 
The three-objective setting also does not improve over the two-objective deployed model, supporting the choice of semantic and dependency as the learned rubrics. 
The FFN variant is highly precise but recall-limited, whereas CRF structure improves recall and overall F1.

\subsection{LongCodeQA Downstream Ablation}
\label{app:longcodeqa_ablation}

\begin{table}[h]
\centering
\small
\renewcommand{\arraystretch}{1.12}
\setlength{\tabcolsep}{5pt}
\caption{
LongCodeQA downstream ablation under 4$\times$ and 8$\times$ compression constraints. 
\method{} achieves the best accuracy and compression in both settings. 
Context-based objectives are competitive at 4$\times$ but unstable under 8$\times$ compression, while removing AST structure substantially reduces 8$\times$ accuracy.
}
\label{tab:longcodeqa_ablation_full}
\begin{tabular}{l c c c c}
\toprule
\multirow{2}{*}{\textbf{Pruner Variant}} 
& \multicolumn{2}{c}{\textbf{4$\times$ Constraint}} 
& \multicolumn{2}{c}{\textbf{8$\times$ Constraint}} \\
\cmidrule(lr){2-3} \cmidrule(lr){4-5}
& $1/\tau$ $\uparrow$ & Acc. $\uparrow$ 
& $1/\tau$ $\uparrow$ & Acc. $\uparrow$ \\
\midrule
Semantic only 
& 15.92 & 56.25 & 17.27 & 58.55 \\
Semantic + Context 
& 12.13 & 59.26 & 16.75 & 56.48 \\
Context + Dependency 
& 11.73 & 60.36 & 15.06 & 51.82 \\
Semantic + Dependency w/o AST 
& 12.99 & 59.09 & 18.00 & 55.56 \\
Semantic + Dependency + Context 
& 11.87 & 57.27 & 15.00 & 53.21 \\
\rowcolor{gray!15}
\method{} 
& \textbf{16.48} & \textbf{61.26} & \textbf{18.51} & \textbf{60.00} \\
\bottomrule
\end{tabular}
\end{table}

The LongCodeQA ablation shows that the validation trends transfer to a downstream long-code QA task. 
Context-based objectives are not consistently robust under stronger compression. 
Removing AST-aware structure especially hurts under the 8$\times$ constraint, suggesting that structural support becomes more important as the token budget tightens.

\subsection{SWE-QA Conan Ablation}
\label{app:conan_ablation}

\begin{table}[h]
\centering
\small
\renewcommand{\arraystretch}{1.12}
\setlength{\tabcolsep}{4.5pt}
\caption{
SWE-QA Conan ablation with Claude Sonnet 4.5. 
\method{} achieves the highest answer score and the fewest interaction rounds while reducing token usage by 20.5\% relative to the unpruned baseline. 
Context-heavy variants compress more aggressively but reduce answer quality, while removing AST structure lowers score and increases rounds compared with \method{}.
}
\label{tab:conan_ablation_full}
\begin{tabular}{l c c c c c}
\toprule
\multirow{2}{*}{\textbf{Method}} 
& \multicolumn{2}{c}{\textbf{Avg Score}} 
& \textbf{Rounds} 
& \multicolumn{2}{c}{\textbf{Tokens (K)}} \\
\cmidrule(lr){2-3} \cmidrule(lr){4-4} \cmidrule(lr){5-6}
& Value & $\Delta$ 
& Value 
& Value & $\Delta$ \\
\midrule
\multicolumn{6}{c}{\textit{Conan, Claude Sonnet 4.5}} \\
\midrule
Unpruned 
& 8.35 & -- & 23.9 & 672.9 & -- \\
+ Semantic only 
& 8.13 & $\downarrow$0.22 & 21.6 & 515.2 & $\downarrow$23.4\% \\
+ Semantic + Context 
& 8.26 & $\downarrow$0.09 & 21.3 & 480.6 & $\downarrow$28.6\% \\
+ Context + Dependency 
& 8.32 & $\downarrow$0.03 & 21.3 & 519.4 & $\downarrow$23.5\% \\
+ Semantic + Dependency w/o AST 
& 8.37 & $\uparrow$0.02 & 21.3 & 494.4 & $\downarrow$26.5\% \\
+ Semantic + Dependency + Context 
& 8.31 & $\downarrow$0.04 & 21.6 & 514.9 & $\downarrow$23.5\% \\
\rowcolor{gray!15}
+ \method{} 
& \textbf{8.48} & $\uparrow$0.13 & \textbf{18.8} & 535.2 & $\downarrow$20.5\% \\
\bottomrule
\end{tabular}
\end{table}

This ablation highlights the difference between aggressive compression and useful compression. 
Semantic+context uses fewer tokens than \method{}, but lowers answer quality. 
Semantic+dependency without AST also uses fewer tokens than \method{}, but requires more rounds and obtains a lower score. 
\method{} retains slightly more context than the most aggressive variants, but the retained context is more useful to the agent, leading to higher score and fewer rounds.

\subsection{Robustness Across Context Constructors}
\label{app:constructor_ablation}

\begin{table}[h]
\centering
\small
\renewcommand{\arraystretch}{1.12}
\setlength{\tabcolsep}{5pt}
\caption{
Robustness across upstream context constructors on LCC 8$\times$. 
\method{} is effective with both RAG and LongCodeZip retrieval. 
The best result is obtained when LongCodeZip is used in rank-only mode, which performs coarse function selection and leaves fine-grained line pruning to \method{}. 
Using LongCodeZip in full compression mode before \method{} substantially hurts performance, showing that stacking two line compressors is harmful.
}
\label{tab:constructor_ablation}
\begin{tabular}{l l c c c}
\toprule
\textbf{Constructor} & \textbf{Pruner} 
& EM $\uparrow$ & ES $\uparrow$ & $1/\tau$ $\uparrow$ \\
\midrule
LCZP (rank-only) & \method{} 
& \textbf{34.5} & \textbf{59.21} & 9.17 \\
LCZP (rank-only) & SWE-Pruner 
& 13.0 & 43.99 & 10.67 \\
\midrule
LCZP (full) & \method{} 
& 24.0 & 52.39 & 10.62 \\
\midrule
RAG ($k=5$) & \method{} 
& 31.0 & 57.27 & 7.53 \\
RAG ($k=5$, wide) & \method{} 
& 31.0 & 58.26 & 6.52 \\
\bottomrule
\end{tabular}
\end{table}

These results support a coarse-to-fine interpretation of \method{}. 
The upstream constructor should preserve candidate recall, while \method{} performs fine-grained semantic/dependency line pruning. 
When the upstream constructor also performs line-level compression, useful evidence may be removed before \method{} can score it.

\paragraph{Summary of ablations.}
Together, the validation, LongCodeQA, and SWE-QA Conan ablations support the same conclusion. 
Semantic relevance identifies direct evidence, but dependency supervision is needed to preserve structural support. 
Context is useful but unstable as a universal learned objective; it is better handled by retrieval, local windows, and rendering. 
CRF-structured prediction improves recall over independent FFN classification, and AST-aware repair improves both validation and downstream performance, especially under tighter compression.

\end{document}